\documentclass[smallextended]{svjour3} 
\smartqed
\RequirePackage{fix-cm}
\usepackage[utf8]{inputenc}
\usepackage[english]{babel}
\usepackage[toc,page]{appendix}

\usepackage{booktabs}
\usepackage{algorithmic}
\usepackage{algorithm}
\usepackage{amsmath}
\usepackage{graphicx}
\usepackage[usenames, dvipsnames]{color}
\usepackage{url}

\DeclareMathOperator*{\argmin}{arg\,min}
\DeclareMathOperator*{\argmax}{arg\,max}

\usepackage{subcaption}
\usepackage{thmtools,thm-restate}

\title{\LARGE \bf
	Uncertainty with UAV Search of Multiple Goal-oriented Targets
}

\author{Mor Sinay \and Noa Agmon \and Oleg Maksimov \and Aviad Fux \and Sarit Kraus}
\institute{
    M.Sinay \at
	Department of Computer Science, Bar-Ilan University, Israel\\
	\email{mor.sinay@gmail.com}
	\and
	N.Agmon \at
	Department of Computer Science, Bar-Ilan University, Israel\\
	\email{agmon@cs.biu.ac.il}
	\and
	S.Kraus \at
	Department of Computer Science, Bar-Ilan University, Israel\\
	\email{sarit@cs.biu.ac.il}
}

\date{Received: date / Accepted: date}

\begin{document}
	\maketitle

	\begin{abstract}
	
		This paper considers the complex problem of a team of UAVs searching targets under uncertainty. The goal of the UAV team is to find all of the moving targets as quickly as possible before they arrive at their selected goal. The uncertainty considered is threefold: First, the UAVs do not know the targets' locations and destinations. Second, the sensing capabilities of the UAVs are not perfect. Third, the targets' movement model is unknown. We suggest a real-time algorithmic framework for the UAVs, combining entropy and stochastic-temporal belief, that aims at optimizing the probability of a quick and successful detection of all of the targets. We have empirically evaluated the algorithmic framework, and have shown its efficiency and significant performance improvement compared to other solutions. Furthermore, we have evaluated our framework using Peer Designed Agents (PDAs), which are computer agents that simulate targets and show that our algorithmic framework outperforms other solutions in this scenario.
	\end{abstract}

	\section{Introduction}\label{Introduction}
    Searching for targets that wish to arrive at a specific goal is an important practical problem that can benefit from the deployment of Unmanned Aerial Vehicles (UAVs).
    Algorithms that identify strategies for the UAVs need to take into consideration the targets' diverse capabilities, movement models, and topological environments where the targets are operating.
    Finding these targets is time-critical; hence, real-time algorithms are necessary.
     
    A simpler version of this challenging problem has been studied in previous work as Minimum Time Search (MTS) \cite{lanillos2012minimum}. In the MTS problem, one or more targets are in unknown locations and need to be found as quickly as possible. Note that in the MTS problem found in the literature \cite{perez2017multi}, the UAVs have perfect detection probability, and the targets are not moving toward a specific goal.
    {\textbf{In this paper}}, we study a complex, realistic model involving real-time with a UAV team facing multiple goal-oriented targets in an unknown location where the targets' movement model is unknown (we know only their initial location). The UAVs have imperfect bounded sensing capabilities traveling in two different environment models: a grid for the UAVs and a graph modeling the roads traversed by the targets.
    
    Initial results of a less complicated problem were presented in previous work \cite{sinay2018uav}. In this paper, we introduce the uncertainty with UAV search Of MultiplE Goal-oriented tArgets (OMEGA). In  OMEGA, the UAVs' goal is to search the targets. The targets are grounded, and their objective is to reach one of the goals.  In OMEGA, multiple targets can enter from various locations. Each target can have a different movement model. A target chooses its goal at random, and this chosen goal is unknown to the UAVs. The target travels in the environment according to a (probabilistic) noisy movement model. In this paper, we present an algorithm to handle two problems: First, when the target movement model strategy is {\em known}. Second, when it is {\em unknown}. The UAVs aim to detect all targets before reaching their goals.
    The uncertainty in our model is threefold: First, the UAVs have imperfect detection capabilities (it can detect a target with some probability $\leq$ 1). Second, the targets' goal selection is probabilistic. Third, the targets' movement model towards their goal is stochastic.
    Our objective is to maximize the probability of detecting all targets before reaching their goals.
    
    OMEGA is a dynamic problem with high dimensionality since the world representation of a grid for the UAVs and a graph for the targets is enormous, in addition to multiple UAVs and targets.
    Hence, the online computation of an optimal solution for the UAVs is intractable for this multi-player problem. Note that the solution is intractable even when only one searcher and one target are involved \cite{trummel1986complexity}.
    
    {\textbf{The main contribution}} of this paper: Providing efficient algorithms for solving realistic OMEGA and MTS problems which include:
    1) A general graph, not a grid, which can cope with environmental constraints such as a no flying zone \cite{perez2017multi};
    2) Goal-oriented targets;
    3) Heterogeneous targets and UAVs;
    4) Multiple targets with different initial locations;
    5) Diverse movement strategies (known and unknown strategies); and 
    6) UAVs with imperfect detection capabilities.
    
    We built a realistic border defense simulation demonstrating its high success rate in the detection of the targets. We empirically evaluated our algorithm with extensive simulations and present its efficiency for targets with {\em known} strategy and {\em unknown} strategy using PDAs.
	
	\section{Related Work}\label{sec:relatedWork}
	Searching for targets when there is incomplete information about their location is a generalization of the MTS problem, where one or more targets are in unknown locations and need to be found as soon as possible.
	There are many approaches for handling the MTS problem under the assumption of a static target \cite{perez2016real,gan2011multi,lanillos2014multi,yang2002decentralized}; however, we are interested in moving targets. 
	
	Lanillos et al. \cite{lanillos2012minimum} address the MTS problem using a Partially Observable Markov Decision Process (POMDP) \cite{monahan1982state} formulation with a single target and a single defender. The target has a Markovian movement model and is not affected by the defender's location, but the target and the defender have the same velocity, and the defender has perfect detection, meaning that if the target is in the defender's detection range, the latter will detect the former. 
	Ru et al. \cite{ru2015distributed} present an algorithm for the MTS problem under the assumptions that the UAVs have uncertainty regarding their location and have restrictions over their movement capability (they can move by 45 degrees at each step).
	In their solution, the environment is represented as a grid with the assumption that each target is located in a different cell of the grid at each time step.
	Perez-Carabaza et al. \cite{perez2017multi} presented a variant of the MTS problem where the defenders have to avoid collisions and added a constraint of places containing {\em no flying zones}. They represented the environment as a grid and presented a heuristic approach for finding a single target using multiple UAVs when the target's movement model is given as a Markovian model.
	Perez-Carabaza et al. \cite{perez2018ant} presented a heuristic based ant colony optimization \cite{dorigo1996ant} for the MTS problem for multiple defenders and a single target. The environment is formed as a grid representation, and the target model is given in advance. However, the UAVs and the target have the same velocity and the UAVs have perfect detection probability, therefore their proposed solution is not applicable for OMEGA.
	
	Since the OMEGA problem can be modeled as a POMDP, we considered the Monte Carlo Tree Search (MCTS) approach for a single initial location. MCTS \cite{coulom2006efficient} is a simulation-based search algorithm for finding optimal strategies. The Information set MCTS (ISMCTS) \cite{cowling2012information} is an extension of MCTS to Imperfect Information games. ISMCTS variations have shown great success in imperfect information games such as poker and hedge. In poker, for example, there is uncertainty about what cards the opponent may have because the cards are dealt face down, in addition to the uncertainty about the opponent’s strategy. In OMEGA, the uncertainty lies in the target movement and its path toward its selected goal, in addition to the UAVs' imperfect detection capabilities. Smooth-UCT \cite{heinrich2015smooth}, an online MCTS algorithm, does not guarantee convergence to Nash-Equilibrium, however it does converge to a sub-optimal strategy much faster than other variations that offer such guarantees (e.g. Online Outcome Sampling \cite{lisy2015online}). For this reason, we implemented the Smooth-UCT in our simulations. In computer poker, the search space is $10^{18}$ \cite{billings2003approximating}, and at least $10$ million simulations are used \cite{heinrich2015smooth} in our settings.
    In our settings, the search space is approximate $10^{92}\approx 480^{35}$ (where 480 is the number of states, and 35 is the average number of steps in each simulation). This suggests that it is more plausible to use a heuristic-based approach to the UAV deployment problem.
	
	We use the entropy measurement \cite{stachniss2005information}, a standard measurement for estimating uncertainty; it is commonly used for problems with incomplete information \cite{vallve2015active,kaufman2016autonomous,blanco2006entropy}. For example, Kaufman et al. \cite{kaufman2016autonomous} introduced an algorithm to explore a grid map using robots where the probability of each cell being either occupied or free was considered. The robot chooses the trajectory that maximizes the map information gain. Blanco et al. \cite{blanco2006entropy} presented an entropy-based algorithm for the robot localization problem. They proposed an approach to measure the certainty of a robot's location based on its previous estimated location.
	
	\section {OMEGA Formulation}
	
	We consider a problem where a team of $m$ UAVs $U = \{u_i\}_{i=1}^{m}$ attempts to detect a set of $n$ targets $A = \{a_i\}_{i=1}^{n}$. The environment is modeled as a directed graph $G = (V,E)$ nested in a 2D environment, representing the roads on which the targets can move, as shown in Fig. \ref{fig:nastedGraph}(a). Each target's objective is to reach a specific goal chosen in advance, where this goal is unknown to the UAVs. Each goal $O_i$ is a subset of the graph $G$. For example, in Fig. \ref{fig:nastedGraph}(b) $O_i$ includes the edges inside the green circle. The entire goal set is denoted by $O=\bigcup\limits_{i=1}^{\bar{k}}O_i$. The UAVs try to protect this set of goals and detect all targets before reaching their goals. The UAVs are familiar with the set $O$ but do not know the specific goal(s) chosen by the targets. 
	Denote the location of a UAV $u_i \in U$ and target $a_i \in A$ at time $t \geq 0$ by $l^t_{u_i} \in E$ and $l^t_{a_i} \in E$, respectively.

    \begin{figure*}[h]
		\begin{subfigure}[h]{0.5\textwidth}
			\centering \includegraphics[clip,width=1.2\columnwidth,keepaspectratio, height=3cm]{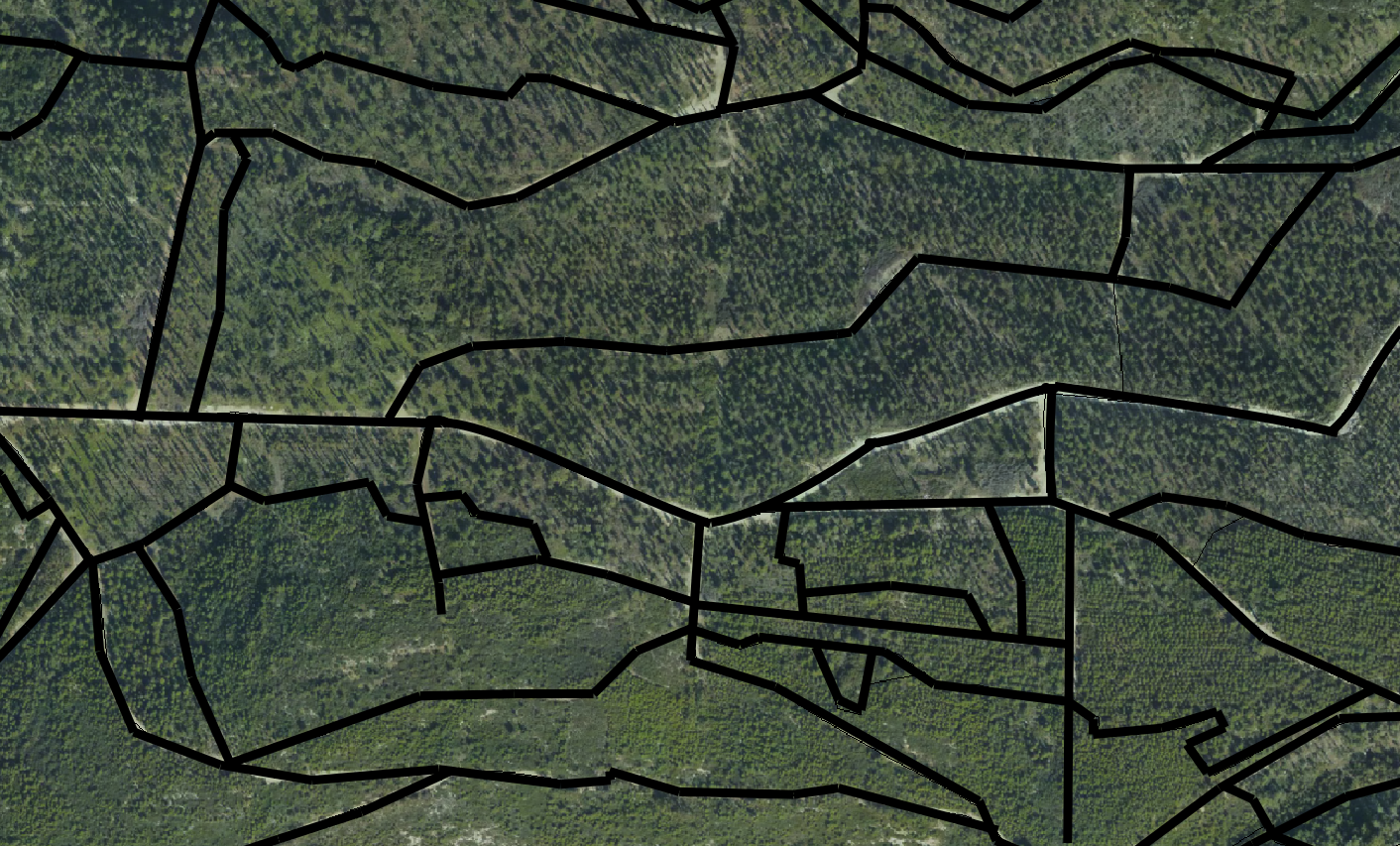}
		    \caption{}
		\end{subfigure}
		\begin{subfigure}[h]{0.5\textwidth}
			\centering
			\includegraphics[clip,width=1.2\columnwidth,keepaspectratio, height=3cm]
			{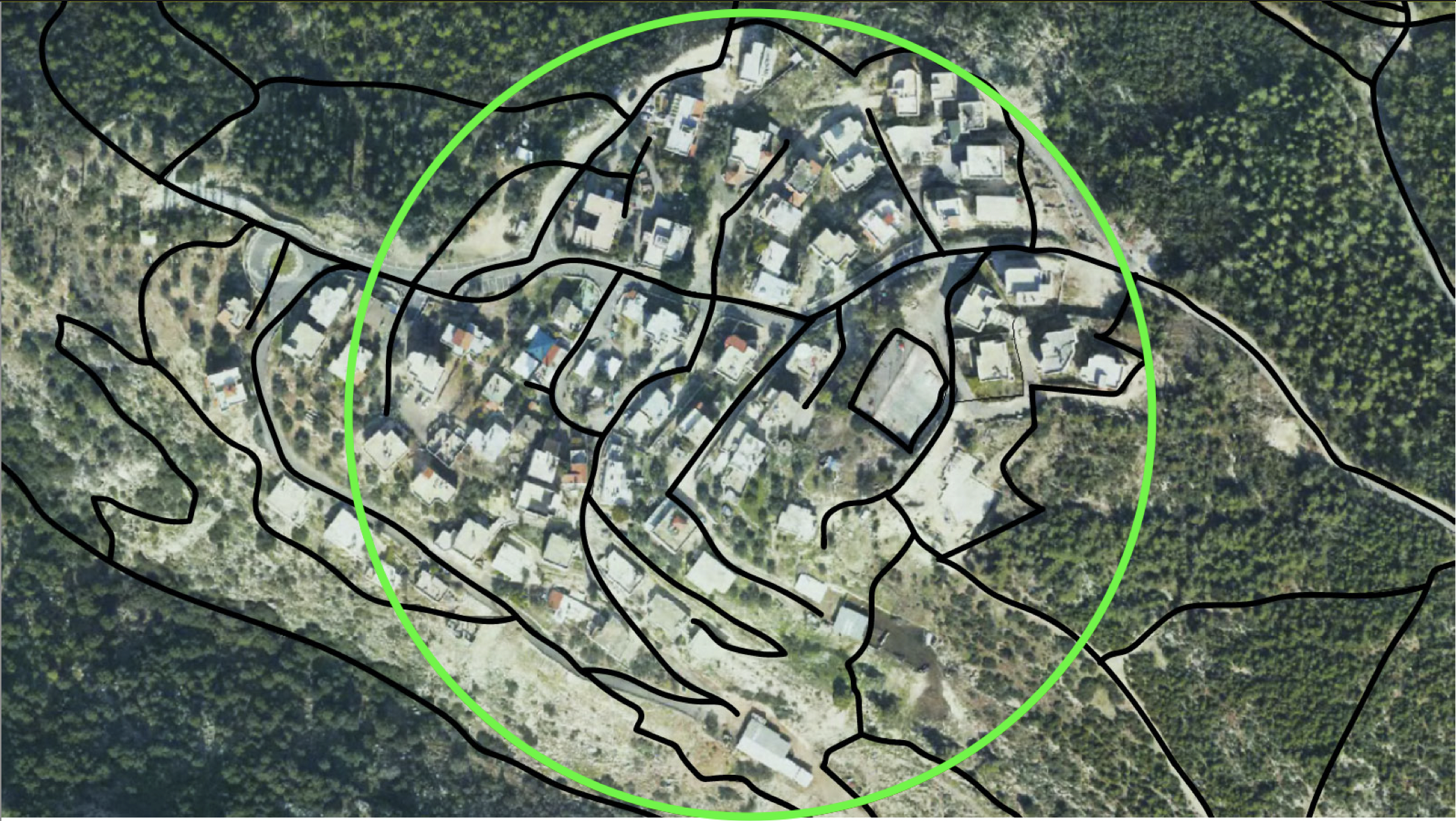}
		\caption{}
		\end{subfigure}
		\caption{Zoom-in view of the graph nested in a 2D environment and a goal example.}\label{fig:nastedGraph}
	\end{figure*}

	\subsection{UAVs' Model}
	We assume the UAVs have limited visibility, hence we denote the detection range (radius) of a UAV $u_i$ as $r_i$, illustrated in Fig. \ref{fig:bird_eye}. Denote the UAVs' minimum detection range by $r = \min\limits_{1\leq i \leq m} \{r_i\}$. We say that a UAV $u_i$ can detect a target if it is within its detection range $r_i$. 
    The UAVs have imperfect detection probability. That is, if a target is within the detection range of a UAV $u_i$, it will detect it with a probability $0 < p_i \leq 1$. Denote the minimum probability of detection by $p = \min\limits_{1\leq i \leq m}p_i$.
	
	\begin{figure}[h]
	\includegraphics[clip,width=1.6\columnwidth,keepaspectratio, height=5cm]{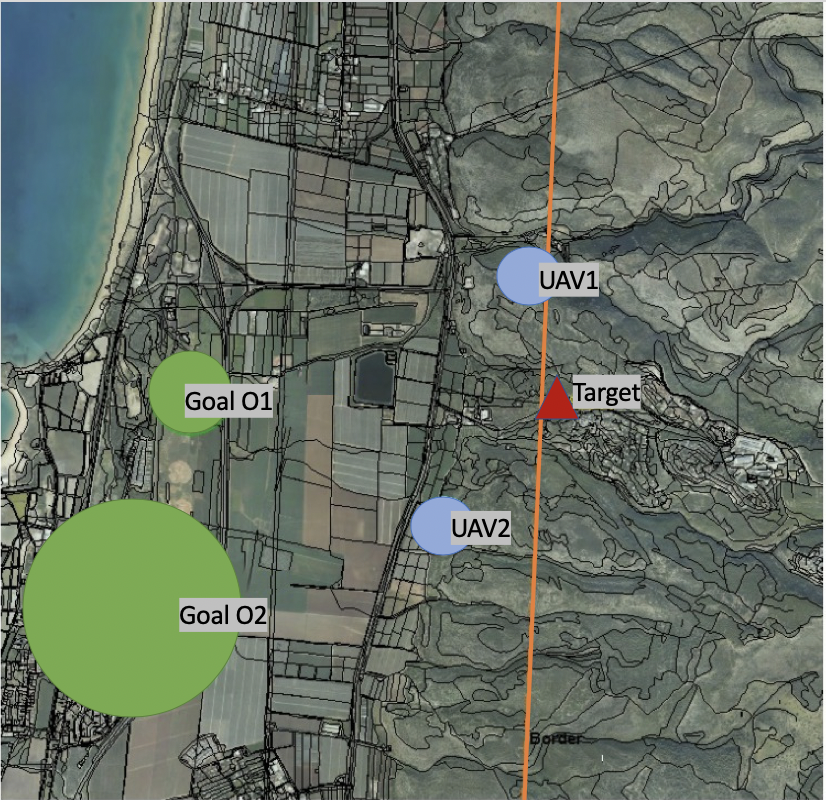}
	{\centering
		\caption{The environment overview from the bird's eye. Circled in green are the goals $O_i \subset G$, circled in blue are the UAVs' detection range. The targets' initial location is the orange line.}\label{fig:bird_eye}}
	\end{figure}

	\subsection{Targets' Model}\label{sub:targetsModel}
	We assume that each target can enter the graph from a set of predefined arbitrary locations denoted by $L$, as demonstrated in Fig. \ref{fig:bird_eye}. Denote the initial location of target $a_i$ by $l^0_{a_i} \in L$. Each target selects its path to its chosen goal $O_i$ in advance. 
	The targets' movement model is unknown. First, we assume that the UAVs know the targets' {\em strategy} (for example, the shortest path). Using the targets' strategy, we estimate a probabilistic belief of the targets' movement model by applying the target strategy from all possible initial locations to all possible goal locations in an offline process. We can only estimate a low accurate movement model from the targets' strategy. Second, we consider the case where the target strategy is unknown. 
	
	\subsection{Problem Formulation}
	Given a team of UAVs and targets as defined above, the problem starts at time $t=0$ when the targets are located at $\{l^0_{a_i}\}_{i=1}^{n} \in L$. All targets move simultaneously according to their movement capabilities and their chosen path. 
    We say that a UAV $u_i$ can detect a target $a_j$ with probability $p_i$, if and only if $\exists t$ such that $l^t_{a_j}$ is within the detection range of $u_i$ which is located at $l^t_{u_i}$ and $l^t_{a_j}\notin O$. Note that after a UAV detects a target, this UAV can now track the target until the target is found by a ground vehicle as discussed in \cite{sinay2018uav}. If the UAV's task is detection, similar to \cite{lanillos2012minimum}, this UAV can continue its flight and be used to detect the remaining targets.
    
    If $\exists t$ such that target $a_j$ arrives at one of the goals $l^t_{a_j}\in O$, then the UAVs lose. Hence, the UAVs' goal is to detect all targets before arriving at their goals.
    Our objective is to determine, for each UAV $u_i \in U$ at each time $t$, a destination for $t+1$ that will maximize the probability that the team of UAVs will detect all targets.
	
	\section{Solving OMEGA}\label{sec:solution}
	
	
	We define the term {\em incoming} set of edges $(v_i, v_{i'})=e_i\in E$ as the set of all edges $\{(v_k, v_i)=e_k \in E\}$ illustrated in Fig. \ref{fig:incoming}.  
	For each target $a_k$ and edge $e_i\in E$, $\bar{M}_k(e_j) = \{(e_j^i,p_j^i)\}$ defines the probability distribution of the incoming edges, that is, the probability $p_j^i$ that target $a_k$ is {\em coming} from edge $e_j^i$ and will choose to {\em move} to edges $\{e_j\}$. Our algorithm uses a general model and applies different movement models to the different targets as presented in Sec. \ref{sec:exp}. 
	
	\begin{figure}[h]
		\centering \includegraphics[clip,width=0.95\columnwidth,keepaspectratio]
		{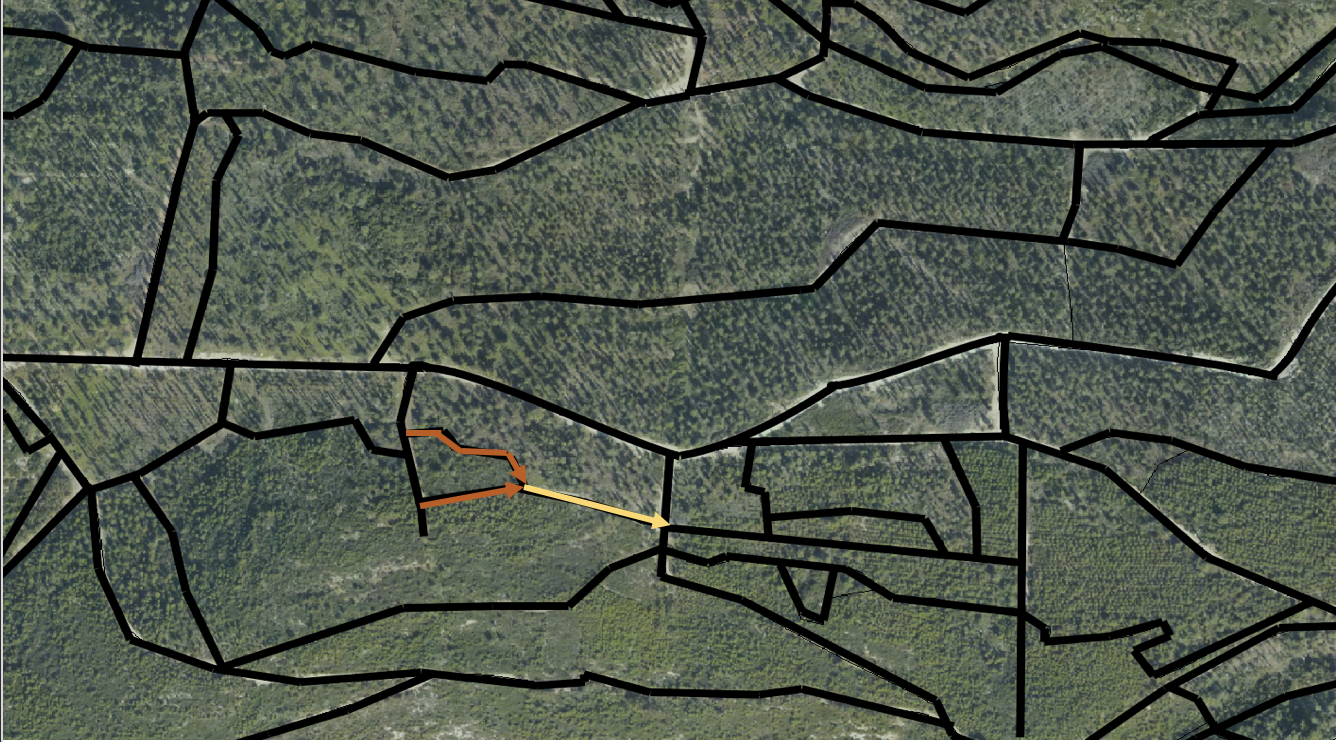}
		\caption{The orange edges are the incoming set of the yellow edge. }\label{fig:incoming}
	\end{figure}
	
	
	The UAVs' task is to detect the targets. However, each target could start from a different location. Hence, we create for each target $a_j \in A$ a probabilistic belief model $P^j(e_i,t)$. That is, our belief as to the probability that a target $a_j \in A$ is located at edge $e_i \in E$ at time $t$, and $\forall t$ $\sum\limits_{e_i \in E} P^j(e_i,t)=1$.
	Note that $P^j(l^0_{a_j},0) = 1$ and $\forall e_i \in E, e_i \neq l^0_{a_j}$ $P^j(e_i,0) = 0$, as $l_{a_j}^0$ is the initial location of the target $a_j$ which is known. The transition from $P^j(e_i,t)$ to $P^j(e_i,t+1)$ is described in Eq. \ref{eq:delta_pt} and depends on our assumption of the targets' movement model $\bar{M}_j(e_i)$. 
	
	\begin{equation}
	\label{eq:delta_pt}
	P^j(e_i,t+1) = \sum\limits_{(e_i^k,p_i^k)\in \bar{M}_j(e_i)} P^j(e_i^k,t)\times p_i^k
	\end{equation}
	
	The UAV is not limited to traveling along the graph. Thus, when a UAV is located on an edge $e$, it will potentially gain information from all edges in $G$ that are within its detection range.
	Therefore, we create a grid representation $C = \{c_i\}_{i=1}^{N}$ layered on top of the graph (see Fig. \ref{fig:cellsAndRoads}). Each cell $c_i$ in the grid is of size $\sqrt{2}r \times \sqrt{2}r$, that is the maximum square that is contained in a circle with radius $r$ (the minimum detection range capability of the UAVs). 
	
	\begin{figure}[h]
		\begin{subfigure}[h]{0.5\textwidth}
			\centering \includegraphics[clip,width=0.95\columnwidth,keepaspectratio]
			{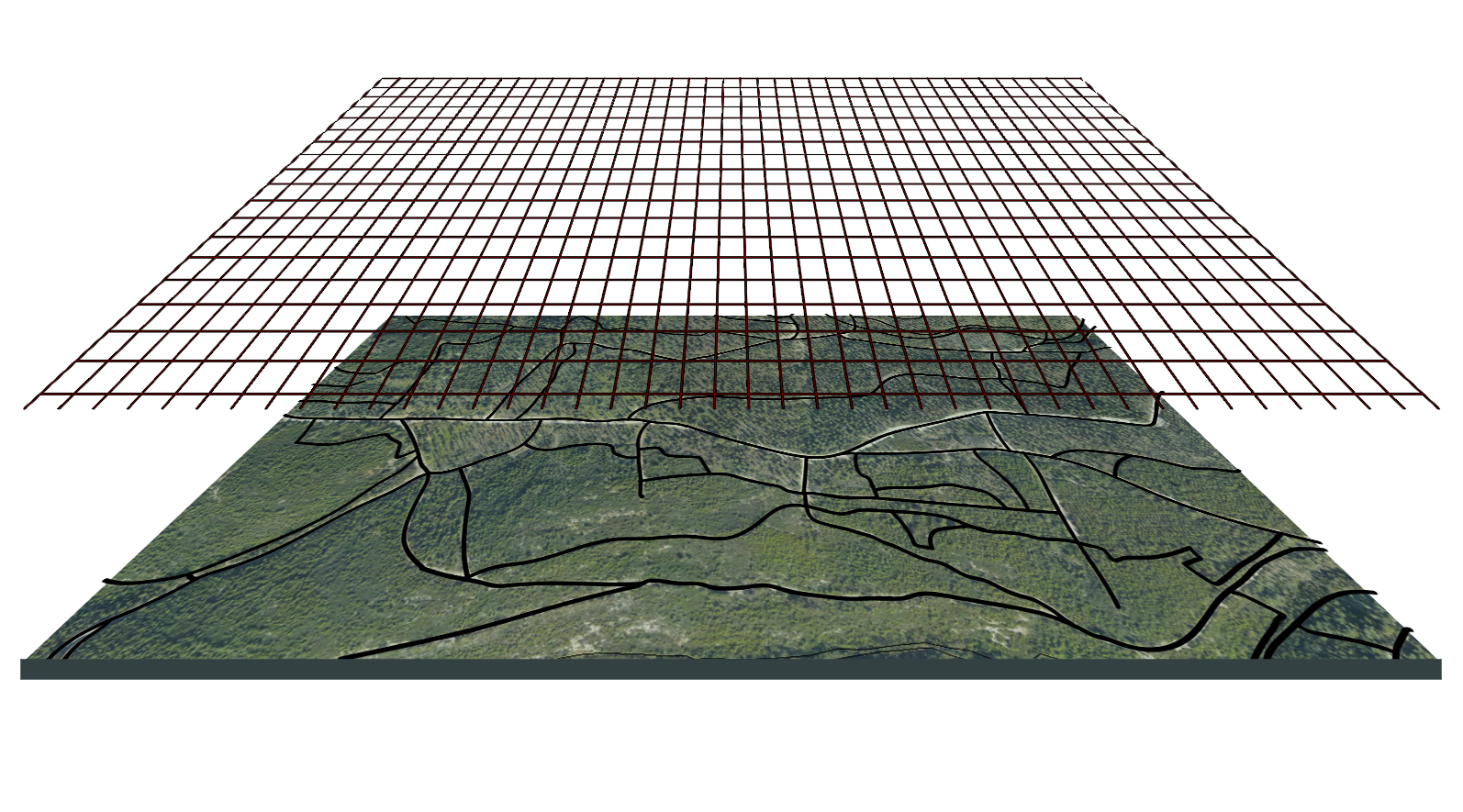}
		\end{subfigure}
		\begin{subfigure}[h]{0.5\textwidth}
			\centering
			\includegraphics[clip,width=0.95\columnwidth,keepaspectratio]
			{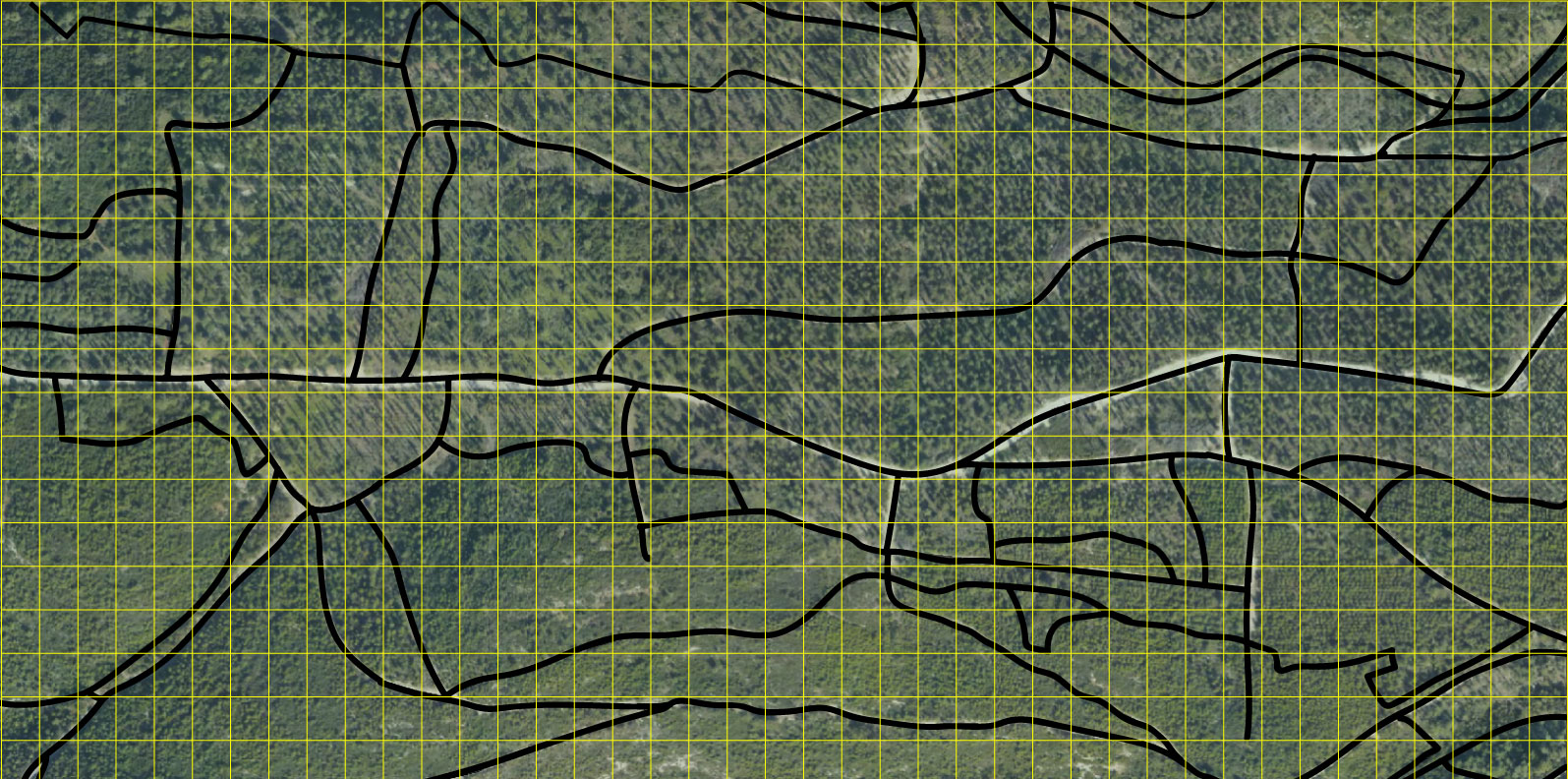}
		\end{subfigure}
		\caption[tree]{A grid layered on top of the graph.}\label{fig:cellsAndRoads}
	\end{figure}

	After constructing the grid and layering it on top of the graph, we will cut the roads (edges of G) into multiple edges (and vertices) so that each edge will be in a single cell.
	
	This offline process produces a new graph $\bar{G} = (\bar{V},\bar{E})$ paired to the grid representation $C = \{c_i\}_{i=1}^N$. 
	We denote by $P^j(c_i,t)$ our belief as to the possibility that a target $a_j$ is located in {\em cell} $c_i \in C$ at time $t$ where $P^j(c_i,t) = \sum\limits_{\bar{e}_k \in c_i} P^j(\bar{e}_k,t)$.
	
	The UAVs' goal is to detect the targets before they arrive at their goals. The UAVs can reduce the uncertainty on the targets' location during the search.
	A search policy for a UAV should specify to which cell to move in order to perform detection of a target.
	We first considered the MCTS approach for a single initial location since the OMEGA problem can be modeled as a POMDP. Given the MCTS's failure to solve the OMEGA problem (see Sec. \ref{sec:relatedWork}), we propose an entropy-based heuristic approach and show its success. 
	
	\subsection{Entropy-Based Approach}\label{subsec:searchTask}
	We first consider the entropy measurement to greedy allocate cells for the UAVs in this task. The entropy measure is based on our belief model, and is defined for each target $a_j$ by:
	\begin{equation}
	\label{eq:ent}
	E^j(C,t) = -\sum\limits_{c_i\in C} P^j(c_i,t)log(P^j(c_i,t))
	\end{equation}
	
	Assume that a UAV is located in cell $c_i$ at time $t$. If the UAV detects a target, then the UAV can continue searching for other targets. However, if the UAV did not detect a target this is not necessarily because the target is not located in cell $c_i$. The probability that a target $a_j$ is at cell $c_i$ and the UAV did not detect it is $1-p_j$. Therefore, if a UAV did not find a target $a_j$ at cell $c_i$ at time $t$, we will update the belief model for each $e_i \in c_i$ according to the probability ($P^j(e_i,t) = (1-p_j)P^j(e_i,t)$). The delta between these two beliefs would split relatively between all cells, similar to a POMDP update. 
	
	The entropy of our belief at time $t$ for target $a_j$ was defined in Eq. \ref{eq:ent}. The temporal entropy $\bar{E}^j(\{c_i\}_{i=1}^m,t)$ returns the new entropy after sending the UAVs to $m$ cells under the assumption that the UAVs did not detect the target $a_j$ (shown in Alg. \ref{alg:newP}).  
	
	\begin{algorithm} [!h]
		\caption{Entropy Update: 
		$\bar{E}^j(\{c_i\}_{i=1}^m,t)$}\label{alg:newP}
		\begin{algorithmic}
			\STATE $\eta = 1 - \sum\limits_{\{c_i\}_{i=1}^m} P^j(c_i,t)\times p_j$
			
			\FORALL {$c_k \in C$}
			\IF {$c_k \in \{c_i\}_{i=1}^m$}
			\STATE $P^j_{temp}(c_k,t) = \frac{P^j(c_k,t)\times (1-p_j)}{\eta}$
			\ELSE
			\STATE $P^j_{temp}(c_k,t) = \frac{P^j(c_k,t)}{\eta}$
			\ENDIF
			\ENDFOR
			
			\RETURN $-\sum\limits_{c_k\in C} P^j_{temp}(c_k,t)log_2(P^j_{temp}(c_k,t))$
		\end{algorithmic}
	\end{algorithm}
	
	Therefore, we define the entropy gain for target $a_j$ as $G^j({\{c_i\}_{i=1}^m},t)$ as the sum of entropy differences. The entropy difference when a UAV detects the target is $p_j\times P^j(c_i,t)(E^j(C,t)-0)$ (because the new entropy is zero). Added to the entropy difference is the probability that the target is in this cell but it did not detect it, or the probability that the target is not in this cell, that is $(1-p_j\times P^j(c_i,t)) (E^j(C,t)-\bar{E}^j(\{c_i\}_{i=1}^m,t))$.
	
	For a given set of $m$ cells the entropy gain is shown in Eq. \ref{eq:gain}.
	
	\begin{equation}
	\label{eq:gain}
	G^j({\{c_i\}_{i=1}^m},t) = E^j(C,t) - \prod \limits_{\{c_i\}_{i=1}^m} (1-p_j\times P^j(c_i,t))\bar{E}^j(\{c_i\}_{i=1}^m,t)
	\end{equation}
	
	We want to maximize the sum of the entropy gain for all targets $a_j \in A$. Hence, we want to find a subset of $m$ cells that maximize Eq. \ref{eq:maxGain}.
	
	\begin{equation}\label{eq:maxGain}
	\begin{split}
	& \argmax\limits_{\{c_i\}_{i=1}^m \in C}\sum\limits_{\{a_j\}_{j=1}^n} G^j({\{c_i\}_{i=1}^m},t) \\
	& = \argmax\limits_{\{c_i\}_{i=1}^m \in C}\sum\limits_{\{a_j\}_{j=1}^n}[ E^j(C,t) 
	- \prod \limits_{\{c_i\}_{i=1}^m} (1-p_j\times P^j(c_i,t))\bar{E}^j(\{c_i\}_{i=1}^m,t) ]\\
	& = \argmin\limits_{\{c_i\}_{i=1}^m \in C} \sum\limits_{\{a_j\}_{j=1}^n} \prod \limits_{\{c_i\}_{i=1}^m} (1-p_j\times P^j(c_i,t))\bar{E}^j(\{c_i\}_{i=1}^m,t)
	\end{split}
	\end{equation}
	
	Entropy is a submodular function \cite{frank1993submodular,madiman2008entropy}, meaning that we can use an online greedy algorithm for a near-optimal solution \cite{feige2006approximation} for determining the best assignments of $m$ UAVs to $m$ cells (without checking all $N\choose m$ combinations of cells).
	
	\begin{restatable}{lem}{maxlemma}
    \label{lem:max}
    When $p = 1$ then the maximum gain (Eq. \ref{eq:maxGain}) returns the cell with the highest probability of being the target's location
	(proof in Appendix).
    \end{restatable}

    If $p < 1$, the maximum entropy gain is not necessarily the maximum cell probability. For example, consider the simple case of one UAV and one target, and two cells $C = \{c_1,c_2\}$, with a probability of $P^1(c_1,t) = 0.9$ and $P^1(c_2,t) = 0.1$ that there is a target located in that cell at time $t$ and $p = 0.9$. In this case, according to Eq. \ref{eq:maxGain} the entropy gain from sending a UAV to $c_1$ is $0.28$ and to $c_2$ is $0.39$. Using {\em only} the entropy measurement would reduce the uncertainy and assign this UAV to $c_2$ although the {\em probability} of having a target there is only $0.1$. 
    However, we are not interested only in reducing the entropy, but in detecting the target. Therefore, we propose an approach that combines maximum gain and maximum probability. We will assign each target a UAV to search the cell that maximizes its detection probability, and all other $m-n$ UAVs will search according to the maximum gain (see Alg. \ref{alg:assi}).

	\begin{algorithm} [h]
		\caption{General Cell Search Assignment}\label{alg:assi}
		\begin{algorithmic}
		\STATE $A = \{ \}$
		\FORALL{$1 \leq j \leq n$}
		    \STATE $A = A \bigcup \argmax\limits_{c_i\in C} P^j(c_i,t)$
		\ENDFOR
		\IF {$|A| \geq m$}
		    \RETURN $m$ cells with max probability
		\ELSE
		    \STATE $k = m-|A|$
			\RETURN $A \bigcup \argmax\limits_{\{c_i\}_{i=1}^k \in C \setminus A}\sum\limits_{\{a_j\}_{j=1}^n} G^j({\{c_i\}_{i=1}^k},t)$
		\ENDIF
		\end{algorithmic}
	\end{algorithm}
	
	When all of the targets enter from the same location, we can look at a light version of this problem. We use only one probabilistic belief model for all of the targets. In all extensive simulation, we found that the balance between choosing cells according to maximum probability and the entropy works better using a threshold, see Alg. \ref{alg:one_enter}. More details on how to choose the threshold $Th^p$ are presented in Sec. \ref{sec:exp}. 
	
	\begin{algorithm} [h]
		\caption{Single Initial Location - Cell Search Assignment}\label{alg:one_enter}
		\begin{algorithmic}
		\IF {$\exists c_j \in C$ such that $P(c_j,t)\geq Th^p$}
			\RETURN $c_j \bigcup G(C \setminus c_j,t)$
		\ELSE
			\RETURN $G(C,t)$
		\ENDIF
		\end{algorithmic}
	\end{algorithm}
	
	\section {Experimental evaluation} \label{sec:exp}
	
	We initially assume that the targets want to arrive at their selected goals as fast as possible before they are detected. Hence, we assume to know their strategy (shortest path). We used an offline process to create a noisy movement model of the targets as follows. We generated a path from each possible initial location edge to any potential goal edge, and created a stochastic Markovian movement model from each edge to its neighbors. Although the strategy is known, there is uncertainty due to the detection probability of the UAVs, the random goal selection of the target, and the low accuracy added from the offline process when computing the probability of all possible locations (initial target location and goal location options).
	
	To evaluate our algorithm we initially used the Simax Smart scenario Generator (SSG)\footnote{\url{www.hartech.co.il}}, a graphical online real-time simulation, which simulates a border protection scenario. This simulation supports multiple homogeneous targets. In this simulation, the UAVs have three goals to protect ($|O|=3$), with a total of $105,000$ roads (edges) in $480$ cells. The velocity of the UAVs is $100 km/h$, and the targets' velocity is $10 km/h$. The targets' initial location is chosen at random with uniform distribution.
    The simulation could not operate in parallel, and the simulations needed to start each time manually. For that reason, we focused on testing our algorithm (Alg. \ref{alg:one_enter}) on the scenario in which all targets enter from the same location. 
    
	We focused on a challenging setting where the UAVs are located {\em far away from the location on the border} when the targets enter. This allows the targets to progress towards their goal before the UAVs begin the search.
	All of the results presented below are the average of at least 40 simulation runs.
	
	Fig. \ref{fig:scout}(a) presents the $Th^p$ that maximizes the success of detecting the target as the UAVs' detection probability changes. As seen, the lower the detection probability is, the higher the chosen threshold is. Meaning that as the uncertainty of the UAVs increases, the more impact the entropy measurement will have on the selection of the cells (rather than moving to the cell with the highest probability). Fig. \ref{fig:scout}(b) shows the influence of the number of UAVs on the detection performance. We considered one, two and three targets (the initial location was chosen randomly with uniform distribution). As the number of UAVs increases, the success rate of detecting the targets increases. In particular, when the number of UAVs is at least the number of attackers +1, then the success rate is perfect. 
    We have compared our solution as presented in Alg. \ref{alg:one_enter} to two variations of our algorithms.
    The first is {\em {Max prob}} that aims at maximizing the probability of immediately locating the targets ($\argmax\limits_{\{c_i\}_{i=1}^k \in C} P(c_i,t)$). The second is the {\em {Entropy}} gain that aims at minimizing the uncertainty by maximizing the entropy gain (see Eq. \ref{eq:maxGain}).
    As seen from Fig. \ref{fig:scout}(c), our solution success rate with one target and one UAV is higher compared to all algorithms (statistically significant using the ANOVA test, with $p-value < 0.05$). 
    
	\begin{figure}[h]
    	\centering
    	\includegraphics[width=1.20\columnwidth,keepaspectratio]{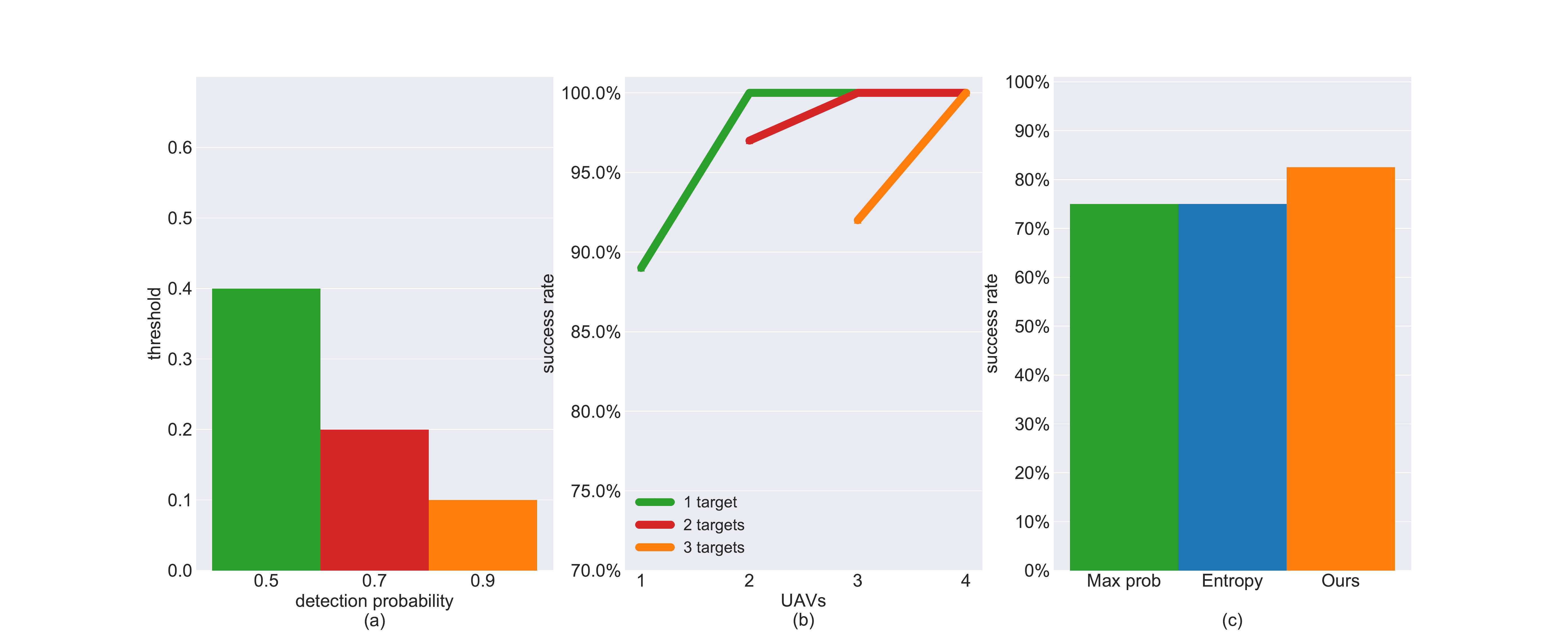}
    	\caption{Simax Smart scenario Generator. \newline 
    	(a): present the threshold $Th^p$ that maximizes the detection probability. \newline
    	(b): presents the success rate of one to three targets and one to four UAVs with $Th^p=0.2$ and detection probability $0.7$. \newline
    	(c): presents the success rate comparison of one target and one UAV with $Th^p=0.2$ and detection probability $0.7$.}
    \end{figure}\label{fig:scout}

	In order to test the algorithm in extensive simulation, we built a border protection simulation of our own\footnote{\url{https://github.com/MorSinay/OMEGA.git}}. In our simulation, we assume that we have the nested roadmap mapped as a graph. The UAVs have seven goals to protect ($|O|=7$) on a total of $7,192$ roads in $480$ cells where $10$ of them are border cells (see Fig. \ref{fig:bird_eye}). The UAVs' velocity is $30 km/h$, and the targets travel at a random velocity between $8-12 km/h$. The targets' initial location and goal selection is chosen randomly. The targets need to travel $12 km$ on average before reaching a goal. We {\em delay} the UAVs from starting to search for the targets until the targets progress at least $6 km$ towards the goal. 
    All of the results presented below are the average of at least 80 simulation runs.
    
	We have compared our solution to two variations of our algorithms.
    The first is {\em {Max avg prob}} that aims at maximizing the average probability of immediately locating the targets ($\argmax\limits_{\{c_i\}_{i=1}^k \in C} \frac{1}{n}\sum\limits_{j=1}^n P^j(c_i,t)$). The second is the {\em {Entropy}} gain that aims at minimizing the uncertainty by maximizing the entropy gain (see Eq. \ref{eq:maxGain}).

   In Fig. \ref{fig:shortest_path}(a) we compare the performance of our solution with two variations of the algorithm using three UAVs and a changing number of targets from two to four with a delay of approximately $7 km$.  The clear advantage of our solution is when there are more UAVs than targets because we gain from both maximizing the detection probability and minimizing the uncertainty. In Fig. \ref{fig:shortest_path}(b) we compare the success rate of our solution on a fixed number of targets (three) with an increasing number of UAVs. The success rate increases as the number of UAVs increases. Fig. \ref{fig:shortest_path}(c) presents the decreasing success rate on a fixed number of UAVs (three) with an increasing number of targets. As we increase the number of targets, it becomes harder for the UAVs to detect all of them before one reaches its goal. 
    
    \begin{figure}[h]
    	\centering
    	\includegraphics[width=1.20\columnwidth,keepaspectratio]{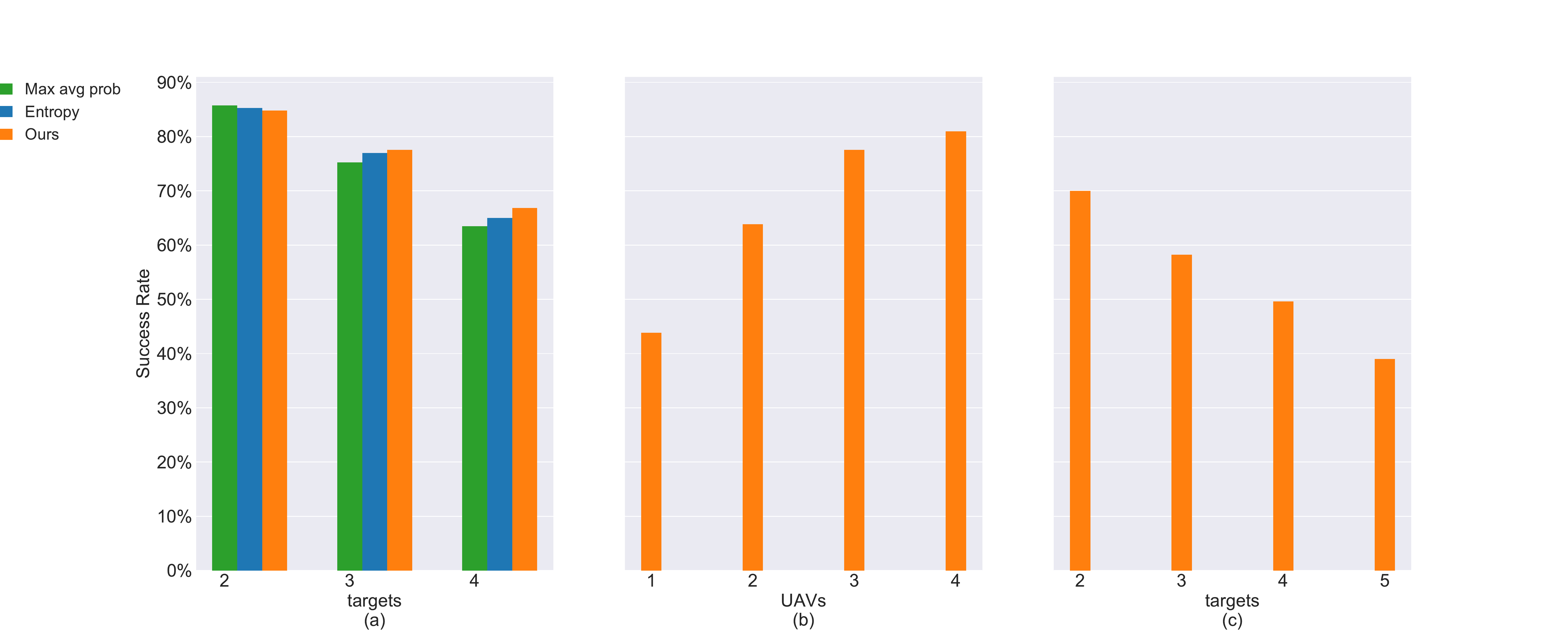}
    	\caption{Success rate of different algorithms on a known target strategy. \newline 
    	(a): presents three UAVs and various targets (delay $7.5 km$). \newline
    	(b): presents three targets and increasing number of UAVs (delay $7 km$). \newline
    	(c): presents the success rate of three UAVs and a multiple number of targets (delay $8 km$).}
    \end{figure}\label{fig:shortest_path}

    In order to test the performance of different targets with different movement models, we used forty PDAs, which are computer agents to simulate the targets. These agents were developed by Computer Science Master's and Ph.D. students as part of a graded exercise of the Advanced Artificial Intelligence course. We gave the PDAs the road map (graph), the initial location and the goals. The students had to create an agent which would generate a path for the target when given an initial location. Some agents, for example, generated the shortest path between the initial location and random goal selection, others used some form of random walk, some moved on side roads. We used thirty different agents (picked randomly) to create our offline movement model and tested the performance on ten agents that were never seen before. For each agent, we produced different paths by generating from every initial location. In those scenarios, there is much more uncertainty in the movement model. Hence, there is a need to give more weight to uncertainty reduction. We minimized the uncertainty (Eq. \ref{eq:maxGain}) in case where there are more undetected targets than UAVs. Otherwise, we used Alg. \ref{alg:assi} which maximizes the detection probability while minimizing uncertainty.
    
	\begin{figure}[t]
		\centering
		\includegraphics[width=1.20\columnwidth,keepaspectratio]{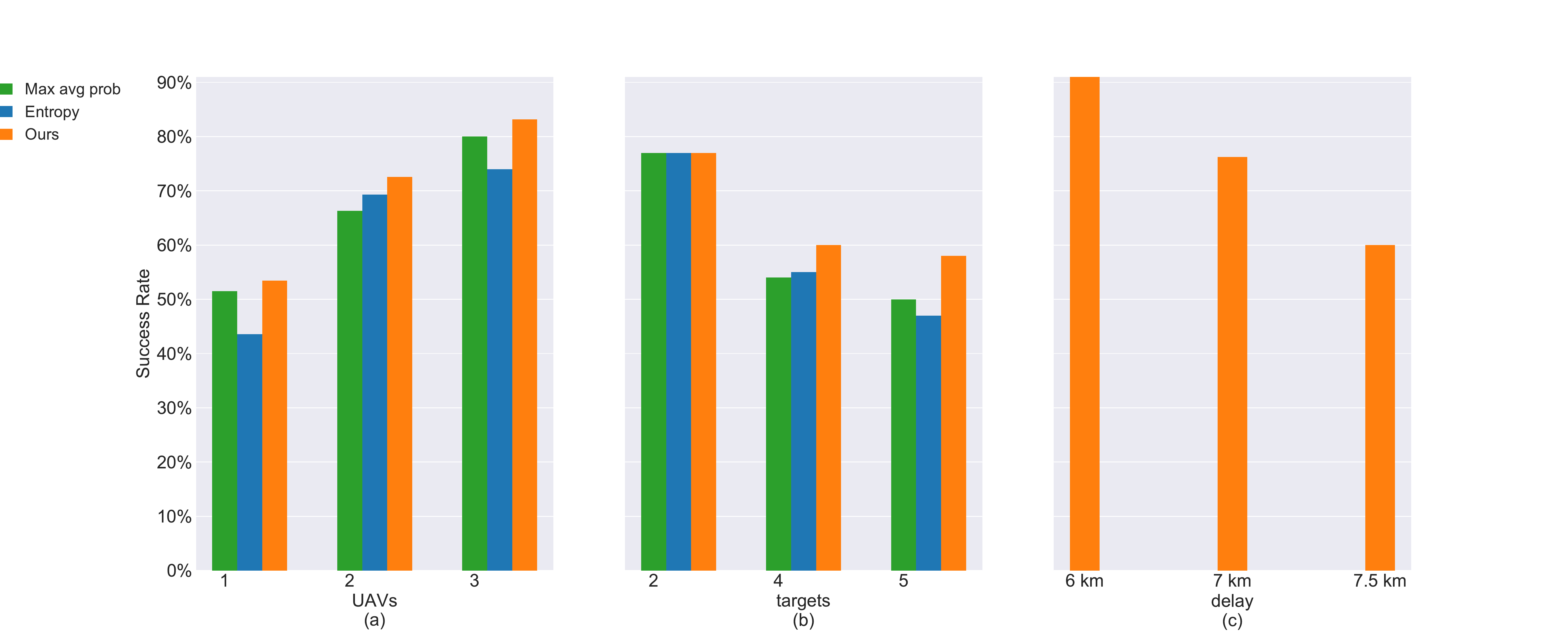}
		\caption{Success rate of different algorithms on an unknown target strategy. \newline 
    	(a): presents three targets and an increasing number of UAVs (delay $7 km$). \newline
    	(b): presents three UAVs and various targets (delay $7.5 km$). \newline
    	(c): presents the success rate of three UAVs and four targets with a different starting delay.}
	\end{figure}\label{fig:students}
	
    In Fig. \ref{fig:students}(a) we compare the performance of our solution with a fixed number of targets and an increasing number of UAVs. Note that even though there is one UAV on three targets, our solution and Entropy have different results due to the algorithm switch when the number of UAVs is equal to the number of undetected targets (statistically significant using the ANOVA test, with $p-value < 0.1$). 
    In Fig. \ref{fig:students}(b) we compare the success rate of our solution using three UAVs with an increasing number of targets (four and five targets are statistically significant with $p-value < 0.05$). In both Fig. \ref{fig:students}(a) and  \ref{fig:students}(b), we benefit from the combination of reducing the uncertainty with the maximum probability. 
    Fig. \ref{fig:students}(c) presents the success rate of a fixed number of targets and UAVs and a growing number of delays. As expected, increasing the delay would decrease the detection performance.

	\section{Conclusion} \label{sec:conclusion}
	In this paper, we introduced the problem of searching targets that wish to arrive at a specific goal under uncertainty (OMEGA). We presented an algorithmic infrastructure for efficiently solving OMEGA and MTS problems for a general graph nested in a 2D environment, assuming imperfect detection of the UAVs on multiple targets with different initial locations and different movement strategies. We have shown an entropy-based algorithm that also refers to the maximum likelihood of finding a target. Our suggested algorithmic solution was tested in a realistic simulation, showing the real-time efficient performance of our framework, also compared to other solutions.
	
\begin{acknowledgements}
This research was supported in part by ISF grant \#1337/15 and by a grant from MOST Israel and the JST Japan.
\end{acknowledgements}

\newpage
\bibliography{bibliography}
\bibliographystyle{spmpsci}

\newpage
\begin{appendices}
\appendix

\section{Proof of lemma~\ref{lem:max}} 
\maxlemma*

\begin{proof}
Assume that for target $a_j$ there are three cells with probability $p_1 = P^j(c_1,t), p_2 = P^j(c_2,t)$ and $p_3 = P^j(c_3,t)$ where $p_3 = 1-p_1-p_2$ and $p_1>p_2>p_3$.
	Denote the current entropy as $E^j$.
	Let's compute the gain from sending a UAV to each cell.

\begin {enumerate}
	\item The new probability belief after selecting to search in cell $c_1$ is:\newline
    	$P^j(c_1, t+1) = 0, P^j(c_2,t+1) = \frac{p_2}{1-p_1}$ and $P^j(c_3,t+1)=\frac{p_3}{1-p_1}$. 
    	\newline The new entropy is:
    	\begin{align*}
        	\bar{E}^j_1 = -\frac{p_2}{1-p_1}log(\frac{p_2}{1-p_1})-\frac{p_3}{1-p_1}log(\frac{p_3}{1-p_1})
    	\end{align*}
    	\newline The gain is:
    	\begin{align*}
        	G^j_1 & = E^j - (1-p_1)\times \bar{E}^j_1 \\
        	& = E^j - (-p_2 log(\frac{p_2}{1-p_1})-p_3 log(\frac{p_3}{1-p_1}))
    	\end{align*}
	
	\item The new probability belief after selecting to search in cell $c_2$ is:\newline
    	$P^j(c_1, t+1) = \frac{p_1}{1-p_2}, P^j(c_2,t+1) = 0$ and $P^j(c_3,t+1)=\frac{p_3}{1-p_2}$. 
    	\newline The new entropy is:
    	\begin{align*}
        	\bar{E}^j_2 = -\frac{p_1}{1-p_2}log(\frac{p_1}{1-p_2})-\frac{p_3}{1-p_2}log(\frac{p_3}{1-p_2})
    	\end{align*}
    	\newline The gain is:
    	\begin{align*}
        	G^j_2 & = E^j - (1-p_2)\times \bar{E}^j_2 \\
        		& = E^j - (-p_1 log(\frac{p_1}{1-p_2})-p_3 log(\frac{p_3}{1-p_2}))
    	\end{align*}
	
	\item The new probability belief after selecting to search in cell $c_3$ is:\newline
    	$P^j(c_1, t+1) = \frac{p_1}{1-p_3}, P^j(c_2,t+1) = \frac{p_2}{1-p_3}$ and $P^j(c_3,t+1) = 0$. 
    	\newline The new entropy is:
    	\begin{align*}
        	\bar{E}^j_3 = -\frac{p_1}{1-p_3}log(\frac{p_1}{1-p_3})-\frac{p_2}{1-p_3}log(\frac{p_2}{1-p_3})
    	\end{align*}
    	\newline The gain is:
    	\begin{align*}
        	G^j_3 & = E^j - (1-p_3)\times \bar{E}^j_3 \\
        		& = E^j - (-p_1 log(\frac{p_1}{1-p_3})-p_2 log(\frac{p_2}{1-p_3}))
    	\end{align*}
	\end{enumerate}
	
	We want to find the cell that maximizes the gain. Hence, we want to show that $G^j_1>G^j_2$ and $G^j_1>G^j_3$ (because we assumed $p_1>p_2>p_3$).
	Let's start by showing that  $G^j_1>G^j_2$.
	\begin{align*}
	& E^j - (-p_2 log(\frac{p_2}{1-p_1})-p_3 log(\frac{p_3}{1-p_1})) > \\
	& E^j - (-p_1 log(\frac{p_1}{1-p_2})-p_3 log(\frac{p_3}{1-p_2})) \\
	\iff \\
	& p_2\times (log(p_2)-log(1-p_1)) + p_3\times (log(p_3)-log(1-p_1)) > \\
	& p_1\times (log(p_1)-log(1-p_2)) + p_3\times (log(p_3)-log(1-p_2)) \\
	\iff 
	& p_2\times (log(p_2)-log(1-p_1)) - p_3\times log(1-p_1) > \\ 
	& p_1\times (log(p_1)-log(1-p_2)) - p_3\times log(1-p_2) \\
	\iff \\
	& p_2\times (log(p_2)-log(1-p_1)) - (1-p_1-p_2)\times log(1-p_1) > \\
	& p_1\times (log(p_1)-log(1-p_2)) - (1-p_1-p_2)\times log(1-p_2) \\	
	\iff 
	& p_2\times log(p_2) - (1-p_1)\times log(1-p_1) > \\
	& p_1\times log(p_1) - (1-p_2)\times log(1-p_2) \\
	\iff 
	& p_2\times log(p_2) + (1-p_2)\times log(1-p_2) > \\
	& p_1\times log(p_1) + (1-p_1)\times log(1-p_1) \\
	\iff 
	& -p_2\times log(p_2) - (1-p_2)\times log(1-p_2) < \\
	& -p_1\times log(p_1) - (1-p_1)\times log(1-p_1) \\
	\iff 
	& H_b(p_2) < H_b(p_1) \text{ where }H_b(x) = -xlog(x)-(1-x)log(1-x)\\ 
	\end{align*}

	\begin{figure}[H]
		\centerline{\includegraphics[width=6cm,height=6cm,keepaspectratio]{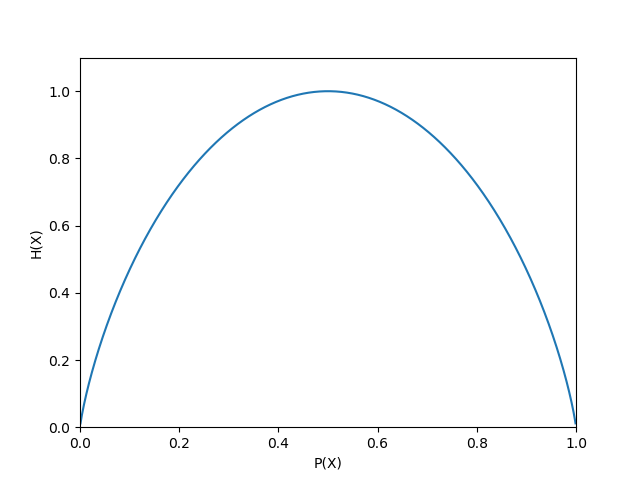}}
		\caption[Binary Entropy]{Binary Entropy ($H_b$)}
		\label{bEntropy}
	\end{figure}
	
	$H_b$ is the binary entropy (see Fig. \ref{bEntropy}). Hence, if $p_2<p_1<\frac{1}{2}$ then $H_b(p_2) < H_b(p_1)$ and the equation holds.
	Let's check in case $p_1>\frac{1}{2}$ and $p_2<\frac{1}{2}$:
	Notice that $H_b(x)$ is an even function around $\frac{1}{2}$. Hence $H_b(\frac{1}{2} - x) = H_b(\frac{1}{2} + x)$.
	We can write $p_1$ as $p_1 = \frac{1}{2} + x$ and $p_2$ as $p_2 = \frac{1}{2} - y$.
	We know that $p_1+p_2<1 \rightarrow \frac{1}{2} + x + \frac{1}{2} - y < 1 \iff x<y$.
	
	We want to show that $H_b(p_2) < H_b(p_1) \iff H_b(\frac{1}{2} - y) < H_b(\frac{1}{2} + x) = H_b(\frac{1}{2} - x)$ so we need to show that $H_b(\frac{1}{2} - y) < H_b(\frac{1}{2} - x) \iff \frac{1}{2} - y < \frac{1}{2} - x \iff -y < -x \iff y > x$ and we know that this holds.
	
    Similar to that we can show that $G^j_1>G^j_3$.
	
	We want to elaborate this for $n$ cells with probability. Assume $n>3$ cells with non-zero probability $p_1>p_2>..>p_n = 1-p_1-p_2-...-p_{n-1}$.
	
	\begin {enumerate}
	\item The new probability belief after selecting to search in cell $c_1$ is:\newline
	$P^j(c_1, t+1) = 0$ and $\forall 1< i\leq n, P^j(c_i,t+1)=\frac{p_i}{1-p_1}$. 
	\newline The new entropy is:
	\begin{align*}
	\bar{E}^j_1 & = -\frac{p_2}{1-p_1}log(\frac{p_2}{1-p_1})- \\ 
		& \frac{p_3}{1-p_1}log(\frac{p_3}{1-p_1}) - .. -\frac{p_n}{1-p_1}log(\frac{p_n}{1-p_1})
	\end{align*}
	\newline The gain is:
	\begin{align*}
	    G^j_1 & = E^j - (1-p_1)\times \bar{E}^j_1 \\
		& = E^j-(-p_2 log(\frac{p_2}{1-p_1})-p_3 log(\frac{p_3}{1-p_1}) - .. -p_n log(\frac{p_n}{1-p_1}))
	\end{align*}
	
	\item The new probability belief after selecting to search in cell $c_2$ is:\newline
	$P^j(c_2, t+1) = 0$ and $\forall 1\leq i\leq n, i \neq 2, P^j(c_i,t+1)=\frac{p_i}{1-p_2}$. 
	\newline The new entropy is:
	\begin{align*}
    \bar{E}^j_2 & = -\frac{p_1}{1-p_2}log(\frac{p_1}{1-p_2})- \\ 
		& \frac{p_3}{1-p_2}log(\frac{p_3}{1-p_2}) - .. -\frac{p_n}{1-p_2}log(\frac{p_n}{1-p_2})
	\end{align*}
	\newline The gain is:
	\begin{align*}
	   G^j_2 & = E^j - (1-p_2)\times \bar{E}^j_2 \\
		& = E^j-(-p_1 log(\frac{p_1}{1-p_2})-p_3 log(\frac{p_3}{1-p_2}) - .. -p_n log(\frac{p_n}{1-p_2}))
	\end{align*}
	\end{enumerate}
	
	We want to find the cell that maximizes the gain. Hence, we want to show that $\forall 1 < i \leq n, G^j_1>G^j_i$.
	Let's start by showing that  $G^j_1>G^j_2$.
	
	\begin{align*}
		& E^j-(-p_2 log(\frac{p_2}{1-p_1})-p_3 log(\frac{p_3}{1-p_1}) - .. -p_n log(\frac{p_n}{1-p_1})) > \\
		& = E^j-(-p_1 log(\frac{p_1}{1-p_2})-p_3 log(\frac{p_3}{1-p_2}) - .. -p_n log(\frac{p_n}{1-p_2})) \\
		\iff 
		& p_2 log(\frac{p_2}{1-p_1})+p_3 log(\frac{p_3}{1-p_1}) + .. +p_n log(\frac{p_n}{1-p_1}) > \\
		& p_1 log(\frac{p_1}{1-p_2})+p_3 log(\frac{p_3}{1-p_2}) + .. +p_n log(\frac{p_n}{1-p_2})\\
		\iff 
		& p_2 (log(p_2)-log(1-p_1)) + p_3 (log(p_3)-log(1-p_1)) + \\ &.. + p_n (log(p_n)-log(1-p_1)) >\\
		& p_1 (log(p_1)-log(1-p_2)) + p_3 (log(p_3)-log(1-p_2)) + \\ &.. + p_n (log(p_n)-log(1-p_2))\\
		\iff \\
		& p_2 (log(p_2)-log(1-p_1)) - p_3 log(1-p_1) - .. - p_n log(1-p_1) >\\
		& p_1 (log(p_1)-log(1-p_2)) - p_3 log(1-p_2) - .. - p_n log(1-p_2)\\
		\iff 
		& p_2 (log(p_2)-log(1-p_1)) - p_3 log(1-p_1) - \\ &.. - (1-p_1-p_2-..-p_{n-1}) log(1-p_1) > \\
		& p_1 (log(p_1)-log(1-p_2)) - p_3 log(1-p_2) - \\ &.. - (1-p_1-p_2-..-p_{n-1}) log(1-p_2)\\
		\iff 
		& p_2 log(p_2) - (1-p_1) log(1-p_1) > \\
		& p_1 log(p_1) - (1-p_2) log(1-p_2) \\
		\iff 
		& p_2 log(p_2) + (1-p_2) log(1-p_2) >\\
		& p_1 log(p_1) + (1-p_1) log(1-p_1)\\
		\iff 
		& -p_2 log(p_2) - (1-p_2) log(1-p_2) <\\
		& -p_1 log(p_1) - (1-p_1) log(1-p_1)\\
		\iff 
		& H_b(p_2) < H_b(p_1)
	\end{align*}
		
	This is the same as before.	
\end{proof}
\end{appendices}

\end{document}